\documentclass{article}

\newcommand{\BPD}{\mathrm{BPD}}
\newcommand{\GPD}{\mathrm{GPD}}

\usepackage{microtype}
\usepackage{graphicx}
\usepackage{subfigure}
\usepackage{booktabs}
\usepackage{hyperref}
\usepackage{chngcntr}
\usepackage{apptools}

\usepackage{amsmath, amsthm, amssymb, pifont, tikz, bbm, bm, microtype, url, mathrsfs, mdframed, enumitem}
\usepackage{float}
\usepackage[accepted]{styles/icml2021}

\newcommand{\mc}[1]{\mathcal{#1}}
\newcommand{\eps}{\varepsilon}
\newcommand{\indic}{\mathbbm{1}}

\newcommand{\ra}{\rightarrow}

\newtheorem{theorem}{Theorem}

\newtheorem{lemma}{Lemma}
\newtheorem{proposition}{Proposition}

\newtheorem{definition}{Definition}

\newmdenv[
  topline=false,
  bottomline=false,
  rightline = false,
  leftmargin=10pt,
  rightmargin=0pt,
  skipabove=16pt, 
  innertopmargin=0pt,
  innerbottommargin=0pt
]{innerproof}

\newenvironment{dproof}[1][Proof]{\begin{proof}[\textbf{\textit{Proof of #1.}}] \text{\vspace{2mm}} \begin{innerproof}}{\end{innerproof}\end{proof}\vspace{8mm}}

\usepackage{etoolbox}
\makeatletter
\patchcmd{\NAT@test}{\else \NAT@nm}{\else \NAT@hyper@{\NAT@nm}}{}{}
\makeatother

\icmltitlerunning{Bad-Policy Density: A Measure of Reinforcement-Learning Hardness}

\begin{document}

\twocolumn[
\icmltitle{Bad-Policy Density: A Measure of Reinforcement Learning Hardness}

\begin{icmlauthorlist}
\icmlauthor{David Abel}{dm}
\icmlauthor{Cameron Allen}{bro}
\icmlauthor{Dilip Arumugam}{st} \\
\icmlauthor{D. Ellis Hershkowitz}{cmu}
\icmlauthor{Michael L. Littman}{bro}
\icmlauthor{Lawson L.S. Wong}{ne}
\end{icmlauthorlist}
\icmlaffiliation{dm}{DeepMind}
\icmlaffiliation{bro}{Department of Computer Science, Brown University}
\icmlaffiliation{st}{Department of Computer Science, Stanford University}
\icmlaffiliation{cmu}{Department of Computer Science, Carnegie Mellon University}
\icmlaffiliation{ne}{Khoury College of Computer Sciences, Northeastern University}
\icmlcorrespondingauthor{David Abel}{dmabel@deepmind.com}

\icmlkeywords{Reinforcement Learning, MDP}
\vskip 0.3in
]
\printAffiliationsAndNotice{}

\begin{abstract}
Reinforcement learning is hard in general. Yet, in many specific environments, learning is easy.
%
What makes learning easy in one environment, but difficult in another?
%
We address this question by proposing a simple measure of reinforcement-learning hardness called the \textit{bad-policy density}. This quantity measures the fraction of the deterministic stationary policy space that is below a desired threshold in value.
%
We prove that this simple quantity has many properties one would expect of a measure of learning hardness. Further, we prove it is NP-hard to compute the measure in general, but there are paths to polynomial-time approximation. We conclude by summarizing potential directions and uses for this measure.
\end{abstract}

\section{Introduction}

%
Markov Decision Processes (MDPs) have long stood as a central model for characterizing the environments that reinforcement-learning agents inhabit. Indeed, the generality of the MDP (and its kin) allows for the description of small and simple environments such as grid worlds, but also large sophisticated ones such as the problem facing a robot organizing books on a shelf or even writing one of the books. A typical objective of research in reinforcement learning (RL) is to develop algorithms that can learn efficiently across the entire space of MDPs. For this reason, it is typically desirable to determine the worst case performance of an RL algorithm across all MDPs of a chosen size, and perhaps, horizon~\cite{Strehl2009,azar2013minimax}.

However, understanding an algorithm's learning efficiency with respect to the size of the MDP misses out on the potentially crucial presence of structure in a given learning problem that can alter the nature and difficulty of learning. Moreover, it is likely that not all MDPs are of interest---oftentimes, relevant subsets of finite MDP space are isolated as being of particular use, such as those with objects~\cite{diuk2008object}, features~\cite{guestrin2001max}, or deterministic transition dynamics~\cite{wen2013efficient}, to name a few. In this sense, it is likely the case that forcing algorithms to perform well on \textit{all} MDPs misses out on important insights that ensure algorithms are well behaved on the MDPs that matter by forcing them to be competent on chaotic environments in which the rapid acquisition of competence \textit{should} be impossible. Indeed, this insight is well established in the literature, with many notable examples establishing extreme efficiency in the presence of structure~\cite{mersereau2009structured,lattimore2014bounded,van2019comments,lattimore2020learning,tirinzoni2020novel}.

%
In this paper, we introduce Bad-Policy Density (BPD) as a simple new measure of RL hardness in finite MDPs. For a given MDP, the BPD measures the fraction of policies in the deterministic policy space that are below a desired threshold ($\tau$) in start-state value. We argue that this measure picks up on interesting structure of MDPs, and can help understand the settings in which RL algorithms might be able to achieve extreme degrees of sample efficiency. We also advocate for this measure as a potential mechanism for determining the difficulty of MDPs in practice---those with higher BPD might be considered more difficult. Further, we suggest that this measure can be useful as a diagnostic tool to assess the impact that priors and structures have on learning hardness, and more generally to identify what characteristics of a task give rise to harder learning.

%
\paragraph{Previous RL Hardness Measures.} Our proposal builds on the insights established by prior hardness measures for RL in finite MDPs (such as the mixing time \cite{kearns2002near}), which we now briefly summarize.
%
First, \citet{maillard2014hard} address the question ``How hard is my MDP?", with the environmental norm, measuring the maximum next-state variance of value throughout the MDP. This measure enables strong theoretical guarantees~\citep{zanette2019tighter} and picks up on an appealing notion of the kinds of MDPs that make RL more difficult---the more costly a mistake might be, the harder time a learning algorithm may have in learning effective behavior in the MDP. However, this measure is precisely zero for all deterministic MDPs, assigning all of them the lowest difficulty achievable under the measure. In this sense, there is room to sharpen our understanding of what constitutes a difficult MDP for RL. 
%
\citet{farahmand2011action} and \citet{bellemare2016increasing} measure hardness through the gap between the $Q^\star$-values of the best and second-best action across state--action pairs; small action gaps induce more challenging problems as an agent requires more samples to reduce estimation error and identify the optimal action.
The eluder dimension~\citep{russo2013eluder,osband2014model,wang2020reinforcement} of a value-function class is a worst-case measure of the maximal number of state--action pairs that must be observed before being able to extrapolate to unseen inputs. Intuitively, if each state--action pair of an MDP yields no information about any others, an agent has no capacity for generalization and is forced into prolonged exploration~\citep{du2019good,van2019comments}.

\citet{jiang2017contextual} propose the Bellman Rank as a suitable measure for the difficulty of Contextual Decision Processes, a generalization of MDPs. Intuitively, the Bellman rank considers the matrix of Bellman residuals induced by a value-function class and creates a link between the behavior policy used to arrive at a particular state--action pair and the value function whose greedy policy defines behavior from that state--action pair. \citet{sun2019model} introduce the witness rank as an analogue to Bellman rank for model-based RL whereas \citet{jin2021bellman} introduce a generalization, the Bellman eluder dimension, as the eluder dimension on the function class of Bellman residuals. One important note on all three measures is their dependence on not only the MDP but also a particular function class; in this way, these measures address the difficulty inherent to learning in the MDP in the specific function class containing the solution. In contrast, BPD focuses on the former source of hardness.

\citet{jaksch2010near} propose the \textit{diameter}---closely related to the span \cite{bartlett2009regal,fruit2018efficient}--- as a measure of MDP hardness, denoting the max number of steps between any two states in the environment. Naturally, a small diameter is suggestive of easier exploration, as the agent may acquire most information about the problem in few steps. Orthogonally, \citet{arumugam2021information} offer information sparsity as an information-theoretic measure of the difficulty of credit assignment within a MDP.
%
With the exception of the environmental norm, all of the aforementioned hardness measures are concerned with characterizing the difficulty of  generalization, exploration, or credit assignment. In contrast, BPD is agnostic to any one particular obstacle to efficient RL and instead simply asks what fraction of the solution space must be eliminated by any agent to solve an MDP, without regard for how efficiently individual agents may prune away sub-optimal solutions.


\section{Bad-Policy Density}

We now introduce and motivate the BPD measure. For a MDP $M = \langle \mc{S}, \mc{A}, R, T, \gamma \rangle$, we call $M$ finite when $|\mc{S}|, |\mc{A}| < \infty$ and assume that $M$ has an initial state $s_0 \in \mc{S}$.

\begin{definition}
The \textbf{Bad-Policy Density} (BPD) of a finite MDP $M$, for a chosen $\tau \in \mathbb{R}$, is given by,
\begin{equation}
        \BPD_{\tau}(M) := \frac{1}{|\Pi_M|} \sum_{\pi \in \Pi_M}\indic\left\{ V^{\pi}(s_0) \leq \tau\right\},
        \label{eq:bad_pol_den}
\end{equation}
for $\Pi_M$ the set of all deterministic mappings from $\mc{S}$ to $\mc{A}$.
\end{definition}

This measure answers a simple question about an MDP: What fraction of deterministic stationary policies are below $\tau$ in start-state value? The BPD has several straightforward properties, which we summarize in the following proposition.

%
\begin{proposition}
\label{prop:bpd_properties}
The BPD satisfies the following properties:
\begin{enumerate}[label=(\roman*)]
    \item For any real $\tau$ and MDP $M$, $0 \leq \BPD_\tau(M) \leq 1$.
    \item For any rational number $q \in [0,1]$, there exists a choice of $\tau$ and deterministic $M_d$ in which $\BPD_\tau(M_d) = q$.
    \item For any MDP $M$, if $\tau_1 < \tau_2$, then  $\BPD_{\tau_1}(M) \leq \BPD_{\tau_2}(M)$.
\end{enumerate}
\end{proposition}

To summarize, this measure is applicable to both deterministic and stochastic MDPs; provides a universal, bounded scale of difficulty (the interval $[0,1]$) thereby allowing normalization-free comparison across MDPs; and ensures monotonic increase as $\tau$ increases. We explore more significant aspects of the measure shortly. 

%
\paragraph{Weaknesses.} There are clear shortcomings to the BPD. First, it is dependent on a potentially arbitrary choice of $\tau$. Across MDPs, it is unclear what the right choice might be. In this sense, a potentially more informative view of the hardness of an MDP is given by the cumulative graph of $\BPD_\tau(M)$, for $\tau \in \left[\textsc{VMin}, \textsc{VMax}\right]$; such a graph will illustrate what region of value space most policies lie in. Second, this measure is not immediately suitable to infinite MDPs. A natural consideration is to replace the enumeration of \autoref{eq:bad_pol_den} with a probability distribution---perhaps, in the assumption-free case, with the uniform distribution. Such a prospect is enticing, but is beyond the scope of this work. Third, the measure fails to capture the impact reward has on the learning dynamics of different RL algorithms. For instance, a shaped reward function that induces the same $\BPD_\tau(M)$ may very well lead to a dramatically easier learning problem for many algorithms. Fourth, the measure makes a commitment as to the significance of the value in $s_0$ (or, more generally, in expectation under a start state distribution). In long horizon tasks, this is not clearly the right perspective to take. Finally, the measure excludes stochastic policies from consideration in favor of focusing on deterministic policies.

In spite of these shortcomings, we take the BPD to serve as a useful and simple measure of RL difficulty in finite MDPs. A natural extension to infinite MDPs measures the probability of sampling a policy below a particular threshold, for some choice of probability distribution over the policy space. We next further establish the usefulness of the measure.

\section{Analysis: Properties of the BPD}

We next present three properties of the BPD: 1) there exists a simple RL algorithm whose episodic sample complexity depends only on the BPD and $\gamma$, and not on $|\mc{S}|$ or $|\mc{A}|$; 2) computing the BPD is NP-hard; 3) there is hope that a polynomial-time approximation is achievable.

\subsection{BPD-Dependent Sample Complexity}

Consider the following algorithmic structure, letting $E$ denote the number of episodes and $H$ denote the horizon:
\begin{align}
    \pi_t &= \texttt{choose}(\Pi_t) \label{eq:alg_choose}\\
    \hat{V}^\pi(s_0) &= \texttt{eval}(M,\pi_t, E, H) \label{eq:alg_run} \\
    \Pi_{t+1} &= \texttt{prune}(\Pi_t, \pi, \tau). \label{eq:alg_prune}
\end{align}

Suppose \texttt{choose} samples a policy uniformly at random, \texttt{eval} evaluates the sampled policy in the MDP, and \texttt{prune} removes $\pi$ if $\hat{V}^\pi(s_0) \leq \tau$ and returns the pruned version of $\Pi$. Let us refer to this simple strategy as the \textsc{PolicySampling} algorithm. We note that the sample complexity of this approach will depend on the $\BPD$ as follows, where $\tau$ inversely defines the magnitude of the mistake bound.
\begin{proposition}
\label{prop:param_sc}
Let $\delta \in (0,1)$ be the desired confidence parameter, $M$ the given MDP, $\eta \geq \textsc{VMax} - \tau$ the mistake threshold, and $\rho = 1-\BPD_\tau(M)$. Then, the episodic \textsc{PolicySampling} algorithm has, with probability $1-\delta$,  sample complexity upper bounded by
\begin{equation}
     \frac{\log_{\rho}\left(\delta\right)\textsc{RMax}^2\ln\left(\frac{2}{\delta}\right)}{2\eta^2(1-\gamma)^3}.
\end{equation}
\end{proposition}

Clearly, this algorithm is entirely impractical in many contexts. However, it illustrates a sense in which the sample complexity of learning \textit{may} depend on the $\BPD$, rather than quantities such as $|\mc{S}|$ and $|\mc{A}|$. Identifying structure-dependent guarantees of this form for existing algorithms is a natural direction for future work. We also note that the \textsc{PolicySampling} algorithm bears a resemblance to sparse sampling \cite{kearns2002sparse}, the PAC bandit approach by \citet{goschin2013planning}, and to the \textsc{Olive} algorithm by \citet{jiang2017contextual}, which iteratively prunes candidate value-function approximators based on the inconsistency of Bellman residuals. 

%
\subsection{Computing BPD is NP-hard}

Naturally, the usefulness of a hardness measure is likely to depend on the practicality in applying it. We first note that computing the number of \textit{optimal} policies is poly-time solvable by assessing the number of actions in each state that yield $Q^*(s,a)$ as follows,
\begin{equation}
    \prod\limits_{s \in \mc{S}} \big|\{a \in \mc{A} : Q^{*}(s,a) = V^*(s)\}\big|.
\end{equation} Unfortunately, when we move to computing the BPD, it is \#P-hard in general. Concretely, we define the problem of computing $\BPD_\tau(M)$ as follows.

\begin{definition}
The \textsc{BPD-Problem} is: \texttt{given} a finite MDP $M$ and a $\tau \in \mathbb{R}$, \texttt{return} $\BPD_\tau(M)$.
\end{definition}

To analyze the difficulty of this problem, we inspect its decision counterpart,

\begin{definition}
The \textsc{BPD-Decision-Problem} is defined as follows: \texttt{given} a finite MDP $M$, $\tau \in \mathbb{R}$, and a proposed level of hardness $\kappa$, \texttt{return} true iff $\BPD_\tau(M) = \kappa$.
\end{definition}

\begin{theorem}
\textsc{BPD-Decision-Problem} is NP-hard.
\label{thm:bpd_np_hard}
\end{theorem}


All proofs are presented in the appendix. As an immediate corollary of the theorem, we note that the counting variant, \textsc{BPD-Problem}, is \#P-hard.

%

\subsection{Approximating the BPD}

In light of the computational intractability of computing the BPD, we instead seek to approximate it. In this section, we give methods for efficiently approximating the good-policy density (GPD), where $\GPD_\tau(M) = 1 - \BPD_\tau(M)$; all of the proofs and algorithms given in this section also apply directly to the bad-policy density but it will be more convenient to describe our results in terms of good-policy density. There are two kinds of approximations for which one might aim: 1) multiplicative approximations; 2) additive approximations. In the former, we can obtain estimates that are within an arbitrary fraction of the true quantity (and thus, are more desirable), while in the latter, we can obtain estimates that are within a chosen $\epsilon$ radius of the true value (which is problematic when the true quantity is near zero or one, as the BPD is likely to be). The additive form is obtained through application of Chernoff inequalities.

\begin{figure*}[h!]
    \centering
    \subfigure[Difficult $N$-Chain]{\includegraphics[width=0.31\textwidth]{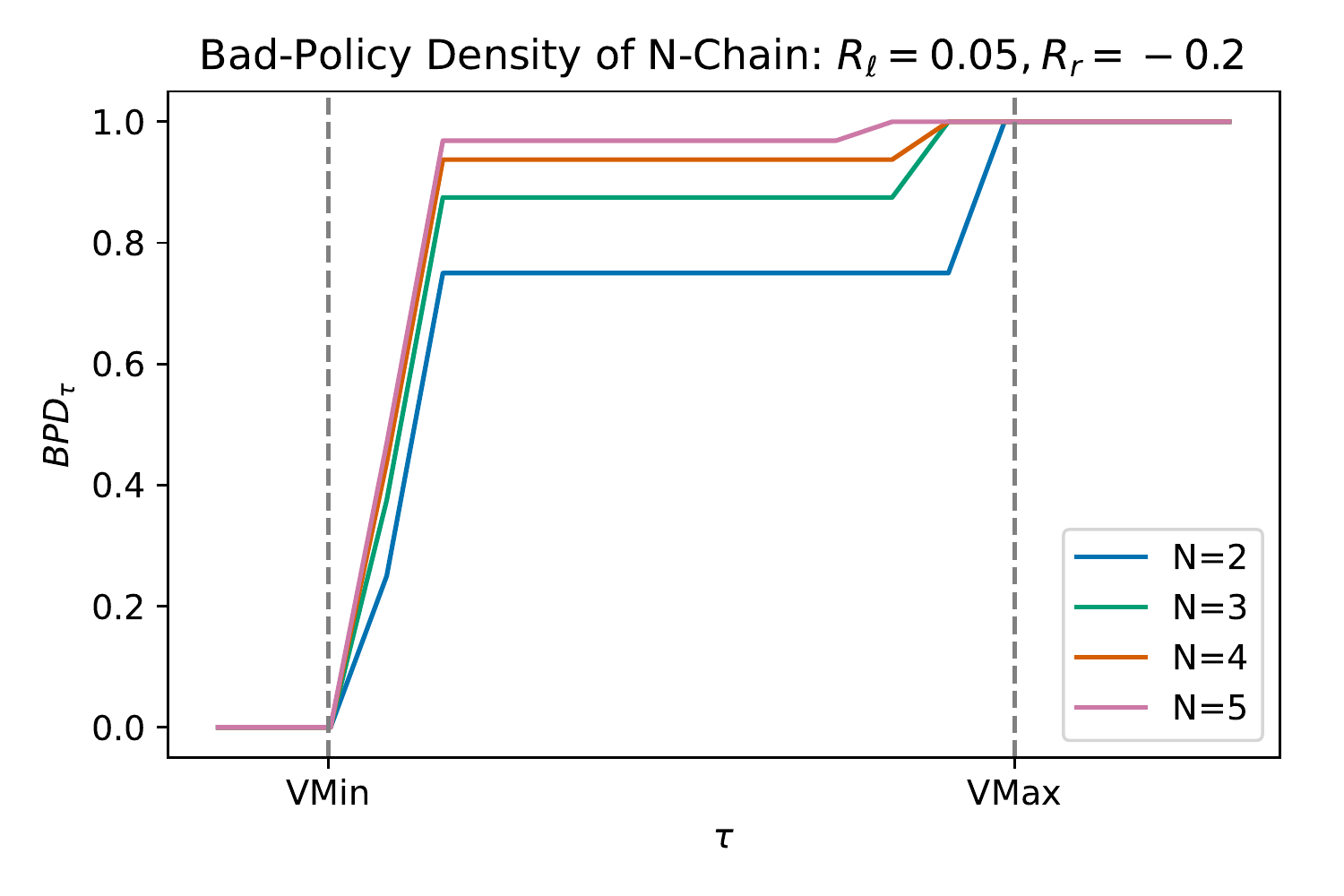}}
    \subfigure[Easy $N$-Chain]{\includegraphics[width=0.31\textwidth]{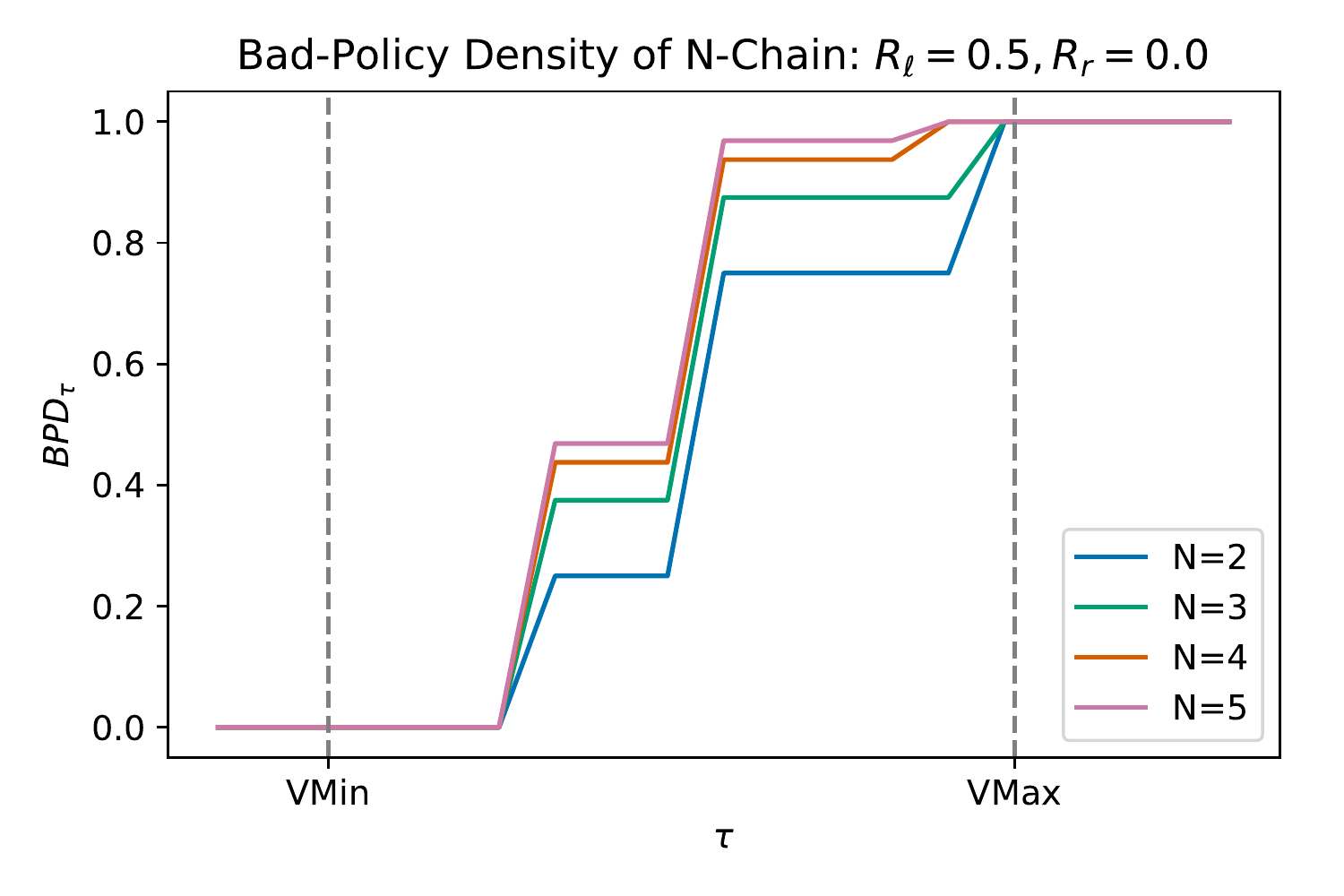}}
    \subfigure[BPD vs. $|\mc{S}|$]{\includegraphics[width=0.31\textwidth]{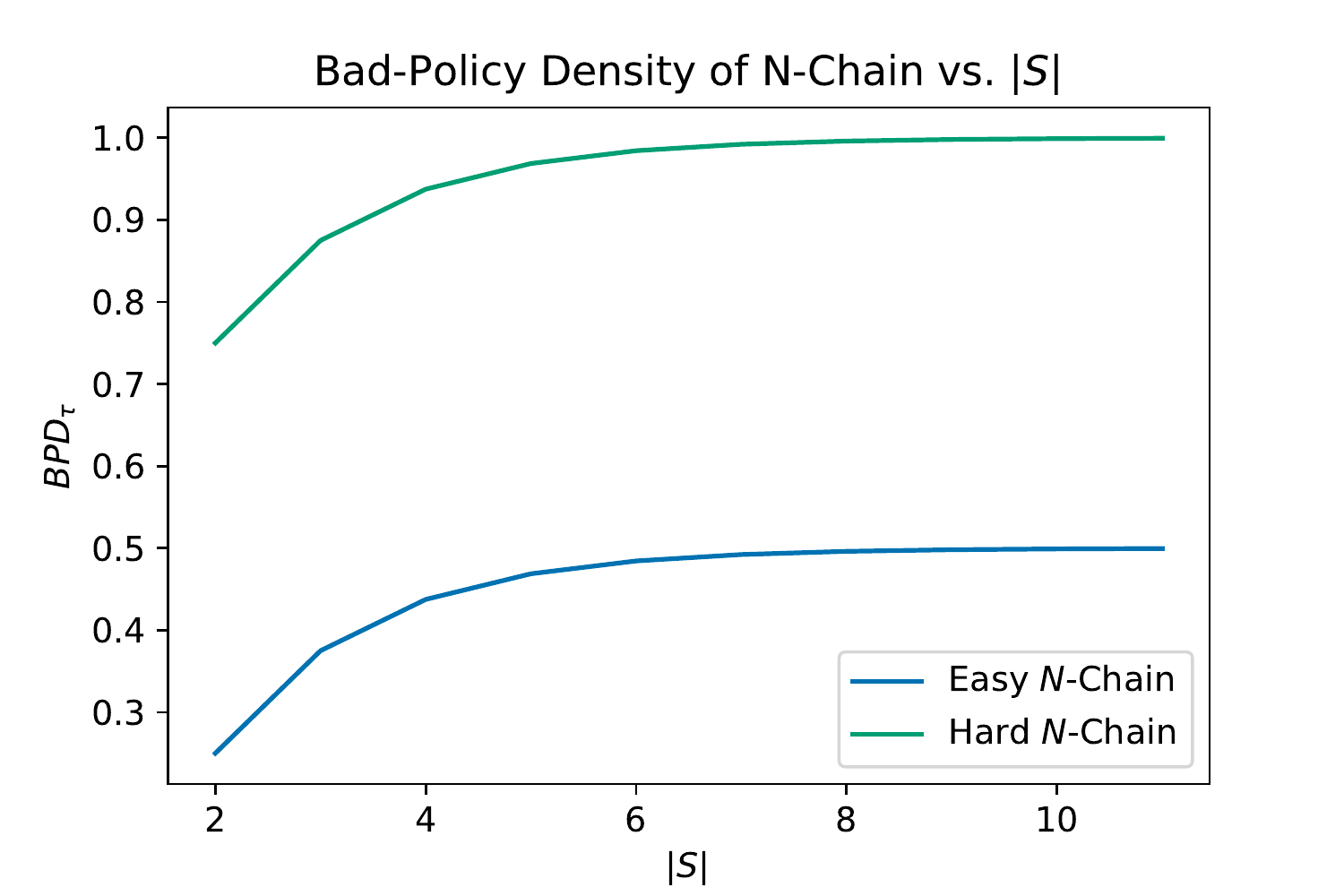}}
    \caption{The BPD in $N$-Chain problems.}
    \label{fig:bpd_chain}
\end{figure*}

\begin{proposition}
\label{prop:bpd_add_approx}
There is a poly-time algorithm that, given a finite MDP $M$, $\tau \in \mathbb{R}$ and constant $\eps > 0$ computes a value $\widehat{\GPD}_{\tau}(M)$ satisfying $|\widehat{\GPD}_{\tau}(M) - \GPD_{\tau}(M)| \leq \eps$ with high probability.\footnote{At least $1 - \frac{1}{\text{poly}(n)}$ where $n$ is the size of the input MDP.}
\end{proposition}

However, achieving the more useful form of multiplicative approximations is more involved. Our next result gives a nearly-complete story as to a path for multiplicative approximation. To actually achieve the result, we first require sampling access to a uniform distribution over ``policies'' which closely resemble policies with at least $\tau$ value on $s_0$. In particular, given some $i \geq 0$ we require access to uniform samples from what we call $\tau$-quality $i$-controlled policies: These are ``policies'' where the agent achieves at least $\tau$ value on $s_0$ but only controls the first $i$ states (for an arbitrary ordering of the states) and is forced to behave optimally on the remaining states. We defer a formal definition of these ``policies'' to the appendix.

\begin{theorem}
There is a poly-time algorithm which, given constant $\eps > 0$ and sample access to a uniform distribution over $\tau$-quality $i$-controlled policies for each $i \geq 0$, returns a value $\widehat{\GPD}_{\tau}(M)$ satisfying $(1 - \eps ) \cdot \GPD_{\tau}(M) \leq \widehat{\GPD}_{\tau}(M) \leq (1 + \eps) \cdot \GPD_{\tau}(M)$ with high probability.
\label{thm:bpd_poly_approx}
\end{theorem}

 We suspect such a distribution can be sampled from in polynomial time, as this procedure is closely related to other polynomial-time constructions. See \citet{jerrum1989approximating} for a similar result on approximating the number of matchings in a graph or \citet{jerrum2003counting} for a comprehensive overview of such approximate counting results.

\section{Discussion}

We conclude with a simple case study of the $\BPD$ in small MDPs, and by suggesting directions for future work.


\begin{figure}[!b]
    \centering
    \includegraphics[width=0.95\columnwidth]{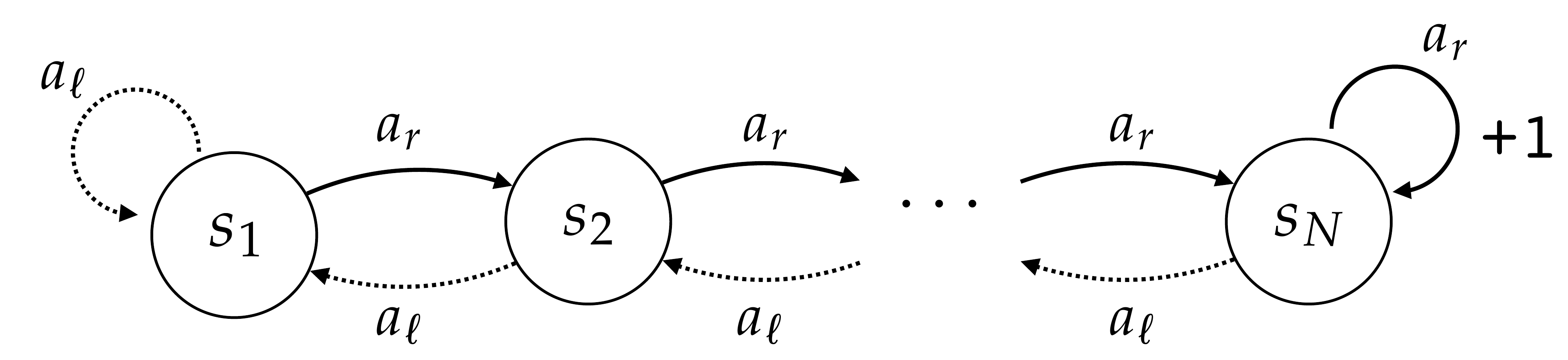}
    \caption{The structure of the $N$-Chain MDP.}
    \label{fig:nchain}
\end{figure}

%
\paragraph{A Small Example.} First, we examine the $\BPD$ of a simple set of MDPs. Because the approximation algorithm described by \autoref{thm:bpd_poly_approx} requires sampling a particular distribution, we here only examine hardness of small MDPs. Specifically, we consider an $N$-Chain graph MDP with $N$ states and $2$ actions, pictured in \autoref{fig:nchain}. The $a_r$ action moves the agent right, and the $a_\ell$ action moves the agent left along the chain. In the rightmost state of the chain, the reward for $a_r$ is $+1$. Otherwise, rewards are determined by constants $R_r$ and $R_\ell$.
%
As shown in \autoref{fig:bpd_chain}, we calculate the $\BPD$ of two variants of $N$-Chain for different settings of $\tau$. In the left figure, we inspect a variation of $N$-Chain where $R_\ell=0.05$ and $R_r=-0.2$, ensuring that the locally greedy action will take $a_\ell$ rather than $a_r$. In contrast, in the center figure, we inspect $N$-Chain where $R_\ell=0.5$ and $R_r=0$. Here, many more policies will be considered good for most values of $\tau$. Indeed, as expected, we find that as we vary $\tau$ from $\textsc{VMin}$ to $\textsc{VMax}$, the hardness of the case on the left sharply increases, whereas this increase is considerably more gradual in (b). In the right figure, we inspect the BPD of the two $N$-chain instances as we vary $N$ for a particular choice of $\tau$. As expected, we find that the harder variant of $N$-chain has considerably higher BPD, and that the impact of increased $N$ is most dramatic early on. Further details and an additional experiment are presented in the appendix.

\paragraph{Future Work.} We foresee many avenues for further research extending the $\BPD$. First, we might remove the BPD's dependence on choice of $\tau$ by instead measuring the \textit{cumulative} $\BPD$ across all choices of $\tau$ ($\int_{\textsc{VMin}}^{\textsc{VMax}} \BPD_\tau(M) d\tau$), rather than with respect to a fixed $\tau$. Analysis and computation of this cumulative quantity are a clear direction for future work. Next, the BPD might be useful to assess the contribution made by different kinds of structures and priors in RL---which ones most dramatically reduce the $\BPD$? The $\BPD$ could be used both as an objective (find the structure that minimizes $\BPD$), and as an evaluation (does structure $X$ or $Y$ most reduce $\BPD$?). Lastly, it is natural to extend the BPD to infinite MDPs, stochastic policies, and perhaps, learning algorithms. 

\section*{Acknowledgements}
The authors would like to acknowledge Will Dabney, Jelena Luketina, Clare Lyle, and Michal Valko for helpful comments and discussions.

\bibliographystyle{styles/icml2021}
\bibliography{ms}

\begin{thebibliography}{31}
\providecommand{\natexlab}[1]{#1}
\providecommand{\url}[1]{\texttt{#1}}
\expandafter\ifx\csname urlstyle\endcsname\relax
  \providecommand{\doi}[1]{doi: #1}\else
  \providecommand{\doi}{doi: \begingroup \urlstyle{rm}\Url}\fi

\bibitem[Arumugam et~al.(2021)Arumugam, Henderson, and
  Bacon]{arumugam2021information}
Arumugam, D., Henderson, P., and Bacon, P.-L.
\newblock An information-theoretic perspective on credit assignment in
  reinforcement learning.
\newblock \emph{arXiv preprint arXiv:2103.06224}, 2021.

\bibitem[Azar et~al.(2013)Azar, Munos, and Kappen]{azar2013minimax}
Azar, M.~G., Munos, R., and Kappen, H.~J.
\newblock Minimax {PAC} bounds on the sample complexity of reinforcement
  learning with a generative model.
\newblock \emph{Machine learning}, 91\penalty0 (3), 2013.

\bibitem[Bartlett \& Tewari(2009)Bartlett and Tewari]{bartlett2009regal}
Bartlett, P.~L. and Tewari, A.
\newblock {REGAL}: A regularization based algorithm for reinforcement learning
  in weakly communicating {MDP}s.
\newblock In \emph{Proceedings of the Conference on Uncertainty in Artificial
  Intelligence}, 2009.

\bibitem[Bellemare et~al.(2016)Bellemare, Ostrovski, Guez, Thomas, and
  Munos]{bellemare2016increasing}
Bellemare, M.~G., Ostrovski, G., Guez, A., Thomas, P.~S., and Munos, R.
\newblock Increasing the action gap: New operators for reinforcement learning.
\newblock In \emph{Proceedings of the AAAI Conference on Artificial
  Intelligence}, 2016.

\bibitem[Diuk et~al.(2008)Diuk, Cohen, and Littman]{diuk2008object}
Diuk, C., Cohen, A., and Littman, M.~L.
\newblock An object-oriented representation for efficient reinforcement
  learning.
\newblock In \emph{Proceedings of the International Conference on Machine
  Learning}, 2008.

\bibitem[Du et~al.(2019)Du, Kakade, Wang, and Yang]{du2019good}
Du, S.~S., Kakade, S.~M., Wang, R., and Yang, L.~F.
\newblock Is a good representation sufficient for sample efficient
  reinforcement learning?
\newblock In \emph{Proceedings of the International Conference on Learning
  Representations}, 2019.

\bibitem[Farahmand(2011)]{farahmand2011action}
Farahmand, A.-m.
\newblock Action-gap phenomenon in reinforcement learning.
\newblock In \emph{Advances in Neural Information Processing Systems}, 2011.

\bibitem[Fruit et~al.(2018)Fruit, Pirotta, Lazaric, and
  Ortner]{fruit2018efficient}
Fruit, R., Pirotta, M., Lazaric, A., and Ortner, R.
\newblock Efficient bias-span-constrained exploration-exploitation in
  reinforcement learning.
\newblock In \emph{Proceedings of the International Conference on Machine
  Learning}, 2018.

\bibitem[Goschin et~al.(2013)Goschin, Weinstein, Littman, and
  Chastain]{goschin2013planning}
Goschin, S., Weinstein, A., Littman, M.~L., and Chastain, E.
\newblock Planning in reward-rich domains via {PAC} bandits.
\newblock In \emph{European Workshop on Reinforcement Learning}, 2013.

\bibitem[Guestrin et~al.(2001)Guestrin, Koller, and Parr]{guestrin2001max}
Guestrin, C., Koller, D., and Parr, R.
\newblock Max-norm projections for factored {MDP}s.
\newblock In \emph{Proceedings of the International Joint Conference on
  Artificial Intelligence}, 2001.

\bibitem[Jaksch et~al.(2010)Jaksch, Ortner, and Auer]{jaksch2010near}
Jaksch, T., Ortner, R., and Auer, P.
\newblock Near-optimal regret bounds for reinforcement learning.
\newblock \emph{Journal of Machine Learning Research}, 11\penalty0
  (Apr):\penalty0 1563--1600, 2010.

\bibitem[Jerrum(2003)]{jerrum2003counting}
Jerrum, M.
\newblock \emph{Counting, sampling and integrating: Algorithms and complexity}.
\newblock Springer Science \& Business Media, 2003.

\bibitem[Jerrum \& Sinclair(1989)Jerrum and Sinclair]{jerrum1989approximating}
Jerrum, M. and Sinclair, A.
\newblock Approximating the permanent.
\newblock \emph{SIAM journal on computing}, 18\penalty0 (6):\penalty0
  1149--1178, 1989.

\bibitem[Jiang et~al.(2017)Jiang, Krishnamurthy, Agarwal, Langford, and
  Schapire]{jiang2017contextual}
Jiang, N., Krishnamurthy, A., Agarwal, A., Langford, J., and Schapire, R.~E.
\newblock Contextual decision processes with low bellman rank are
  {PAC}-learnable.
\newblock In \emph{Proceedings of the International Conference on Machine
  Learning}, 2017.

\bibitem[Jin et~al.(2021)Jin, Liu, and Miryoosefi]{jin2021bellman}
Jin, C., Liu, Q., and Miryoosefi, S.
\newblock Bellman eluder dimension: New rich classes of {RL} problems, and
  sample-efficient algorithms.
\newblock \emph{arXiv preprint arXiv:2102.00815}, 2021.

\bibitem[Kearns \& Singh(2002)Kearns and Singh]{kearns2002near}
Kearns, M. and Singh, S.
\newblock Near-optimal reinforcement learning in polynomial time.
\newblock \emph{Machine learning}, 49\penalty0 (2-3):\penalty0 209--232, 2002.

\bibitem[Kearns et~al.(2002)Kearns, Mansour, and Ng]{kearns2002sparse}
Kearns, M., Mansour, Y., and Ng, A.~Y.
\newblock A sparse sampling algorithm for near-optimal planning in large
  {M}arkov decision processes.
\newblock \emph{Machine learning}, 49\penalty0 (2-3), 2002.

\bibitem[Lattimore \& Munos(2014)Lattimore and Munos]{lattimore2014bounded}
Lattimore, T. and Munos, R.
\newblock Bounded regret for finite-armed structured bandits.
\newblock In \emph{Advances in Neural Information Processing Systems}, 2014.

\bibitem[Lattimore et~al.(2020)Lattimore, Szepesvari, and
  Weisz]{lattimore2020learning}
Lattimore, T., Szepesvari, C., and Weisz, G.
\newblock Learning with good feature representations in bandits and in {RL}
  with a generative model.
\newblock In \emph{Proceedings of the International Conference on Machine
  Learning}, 2020.

\bibitem[Maillard et~al.(2014)Maillard, Mann, and Mannor]{maillard2014hard}
Maillard, O.-A., Mann, T.~A., and Mannor, S.
\newblock ``{H}ow hard is my {MDP}?" the distribution-norm to the rescue.
\newblock In \emph{Advances in Neural Information Processing Systems}, 2014.

\bibitem[Mersereau et~al.(2009)Mersereau, Rusmevichientong, and
  Tsitsiklis]{mersereau2009structured}
Mersereau, A.~J., Rusmevichientong, P., and Tsitsiklis, J.~N.
\newblock A structured multiarmed bandit problem and the greedy policy.
\newblock \emph{IEEE Transactions on Automatic Control}, 54\penalty0 (12),
  2009.

\bibitem[Osband \& Van~Roy(2014)Osband and Van~Roy]{osband2014model}
Osband, I. and Van~Roy, B.
\newblock Model-based reinforcement learning and the eluder dimension.
\newblock In \emph{Advances in Neural Information Processing Systems}, 2014.

\bibitem[Russell \& Norvig(2009)Russell and Norvig]{russell1995modern}
Russell, S. and Norvig, P.
\newblock \emph{Artificial Intelligence: A Modern Approach}.
\newblock Prentice Hall, 2009.

\bibitem[Russo \& Van~Roy(2013)Russo and Van~Roy]{russo2013eluder}
Russo, D. and Van~Roy, B.
\newblock Eluder dimension and the sample complexity of optimistic exploration.
\newblock In \emph{Advances in Neural Information Processing Systems}, 2013.

\bibitem[Strehl et~al.(2009)Strehl, Li, and Littman]{Strehl2009}
Strehl, A.~L., Li, L., and Littman, M.~L.
\newblock Reinforcement learning in finite {MDP}s: {PAC} analysis.
\newblock \emph{Journal of Machine Learning Research}, 10:\penalty0 2413--2444,
  2009.

\bibitem[Sun et~al.(2019)Sun, Jiang, Krishnamurthy, Agarwal, and
  Langford]{sun2019model}
Sun, W., Jiang, N., Krishnamurthy, A., Agarwal, A., and Langford, J.
\newblock Model-based {RL} in contextual decision processes: {PAC} bounds and
  exponential improvements over model-free approaches.
\newblock In \emph{Proceedings of the Conference on Learning Theory}, 2019.

\bibitem[Tirinzoni et~al.(2020)Tirinzoni, Lazaric, and
  Restelli]{tirinzoni2020novel}
Tirinzoni, A., Lazaric, A., and Restelli, M.
\newblock A novel confidence-based algorithm for structured bandits.
\newblock In \emph{Proceedings of the International Conference on Artificial
  Intelligence and Statistics}, 2020.

\bibitem[Van~Roy \& Dong(2019)Van~Roy and Dong]{van2019comments}
Van~Roy, B. and Dong, S.
\newblock Comments on the {D}u-{K}akade-{W}ang-{Y}ang lower bounds.
\newblock \emph{arXiv preprint arXiv:1911.07910}, 2019.

\bibitem[Wang et~al.(2020)Wang, Salakhutdinov, and Yang]{wang2020reinforcement}
Wang, R., Salakhutdinov, R.~R., and Yang, L.
\newblock Reinforcement learning with general value function approximation:
  Provably efficient approach via bounded eluder dimension.
\newblock \emph{Advances in Neural Information Processing Systems}, 2020.

\bibitem[Wen \& Van~Roy(2013)Wen and Van~Roy]{wen2013efficient}
Wen, Z. and Van~Roy, B.
\newblock Efficient exploration and value function generalization in
  deterministic systems.
\newblock In \emph{Advances in Neural Information Processing Systems}, 2013.

\bibitem[Zanette \& Brunskill(2019)Zanette and Brunskill]{zanette2019tighter}
Zanette, A. and Brunskill, E.
\newblock Tighter problem-dependent regret bounds in reinforcement learning
  without domain knowledge using value function bounds.
\newblock In \emph{Proceedings of the International Conference on Machine
  Learning}, 2019.

\end{thebibliography}

\appendix

\section{Proofs}

We first provide proofs of central results.

\begin{dproof}[\autoref{prop:bpd_properties}]

First, simply note that by definition $\BPD_\tau(M)$ is minimal when \textit{all} policies are good (so $V^{\pi}(s_0) > \tau, \forall_{\pi \in \Pi_M}$, yielding $\BPD_\tau(M) = 0$. By the same reasoning, when the value of every policy is greater than $\tau$, we find $\BPD_\tau(M)$ is maximized at $1$. \\

Second, note that for any choice of rational $q \in [0,1]$, there must exist integers $z_1$ and $z_2$ such that $q = \frac{z_1}{z_2}$. Since $q$ can be zero, but non-negative, we note that $z_1$ and $z_2$ must take on non-negative quantities, too. 
Note that for any $q$, there exists an $x$ and a finite MDP with $|\mc{S}|$ and $|\mc{A}|$ such that $\frac{x}{|\Pi_M|} = \frac{x}{|\mc{A}|^{|\mc{S}|}}$ (in the trivial case, $|\mc{S}|=1$, and $|\mc{A}|$ can be any natural number). Then, choose $T$, $R$, and $\gamma$ such that $|\{\pi \in \Pi_M : V^{\pi}(s_0) \geq \tau\}| = x$. Note again that this may achieved in the trivial case when $|\mc{S}| = 1$, and $R(s_0,\pi_i(s_0)) > \tau (1-\gamma)$, for each $\pi_i$ in the good set, and $R(s_0,\pi_i(s_0)) \leq \tau (1-\gamma)$ for each $\pi_i$ in the bad set. \\

Third, note that as $\tau$ increases, the number of policies that are considered bad cannot decrease, as any policy where $V^{\pi}(s_0) \leq \tau$ will also adhere to $V^{\pi}(s_0) \leq \tau + \epsilon$, for arbitrarily small positive $\epsilon$. \qedhere
\end{dproof}

\begin{dproof}[\autoref{prop:param_sc}]
By Hoeffding's inequality, we bound the number of episodes of a given policy $\pi$ needed to obtain an $\eps \in \mathbb{R}$ accurate estimate of $V^{\pi}(s_0)$ as follows,
\begin{equation}
    \Pr\{|\hat{V}^\pi(s_0) - V^\pi(s_0)| \leq \eps\} \geq 1-\delta_1,
\end{equation}
with $\delta_1 = 2\exp\left(-\frac{2 m^2 \eps^2}{m\textsc{VMax}^2}\right)$. We let $\eps = \eta \geq \textsc{VMax} - \tau$, for the chosen $\tau$. Moreover we run each policy for $h = \frac{1}{1-\gamma}$ steps per episode to yield the estimate $\hat{V}$. \\

Next, note that BPD induces a geometric distribution on the number of sampled policies needed to sample at least one good policy. Specifically we can bound the error probability as:
\begin{align}
    (1- \BPD_\tau(M))^{k} &\leq \delta_2,
\end{align}
for $k$ the number of evaluated policies, and $\delta_2 \in (0,1)$ a confidence parameter. Then, for
\begin{equation}
    k = \log_{1-\BPD_\tau}\left(\delta_2\right)
\end{equation}

Thus, after $k m h$ time-steps, we will find a policy with $V^{\pi}(s_0) \geq \tau$ with high probability, where:
\begin{itemize}
    \item $k$ is the number of policies we need to sample.
    \[
    k \leq \log_{1-\BPD_\tau(M)}\left(\delta_2\right).
    \]
    
    \item $m$ is the number of episodes we need to evaluate each policy.
    \[
    m \leq \frac{\textsc{VMax}^2\ln\left(\frac{2}{\delta_1}\right)}{2\eta^2}
    \]
    
    \item $h$ is the number of steps per-episode we let each policy run,
    \begin{equation}
        h = \frac{1}{1-\gamma}.
    \end{equation}
\end{itemize}

Hence, letting the mistake bound $\eta \geq V^*(s_0) - \tau$,
\begin{align}
    k m h &\leq \log_{1-\BPD_\tau(M)}\left(\delta_2\right) \frac{\textsc{VMax}^2\ln\left(\frac{2}{\delta_1}\right)}{2\eta^2(1-\gamma)}, \\
    &= \log_{1-\BPD_\tau(M)}\left(\delta_2\right) \frac{\textsc{RMax}^2\ln\left(\frac{2}{\delta_1}\right)}{2\eta^2(1-\gamma)^3}.
\end{align}\qedhere
\end{dproof}

\begin{figure}[t!]
    \centering
    \includegraphics[width=0.65\columnwidth]{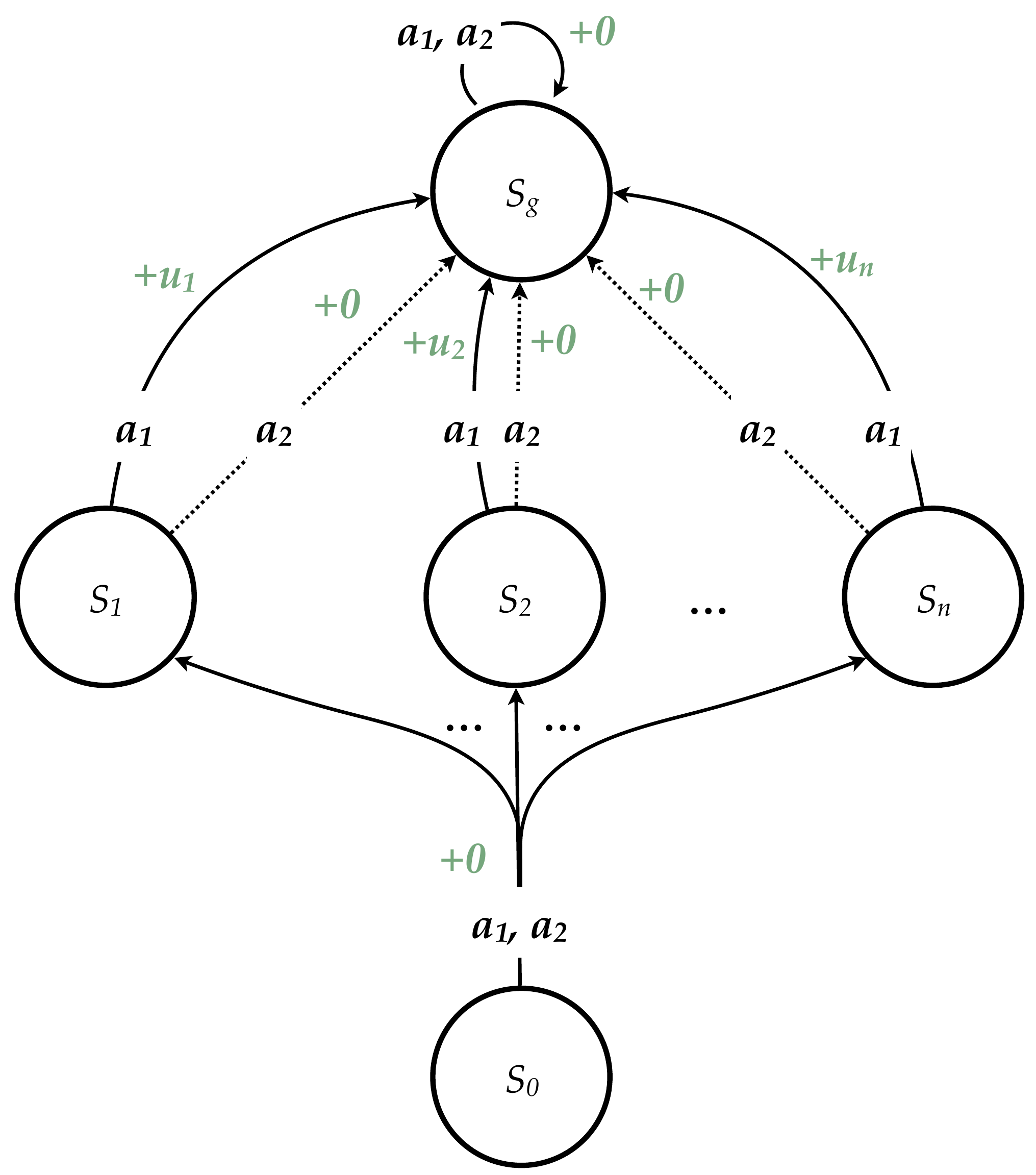}
    \caption{Given an instance of subset sum with $\mc{U} = \{u_1, \ldots, u_n\}$, we construct the above MDP and use a BPD solver to solve the subset sum instance.}
    \label{fig:ss_mdp_np_hard}
\end{figure}
\begin{dproof}[\autoref{thm:bpd_np_hard}]
Recall that an instance of the \textsc{SubsetSum} problem is given by a finite, non-empty set of $N$ non-negative integers $\mc{U} = \{u_1,u_2,\ldots,u_N\}$ and a target value $t$. In the decision version of the \textsc{SubsetSum} problem, we ask if there exists a subset $X \subseteq \mc{U}$ such that $\sum_{x \in X} x = t$. Recall that this problem is NP-complete. \\

Given an instance of \textsc{SubsetSum} as defined above, we construct an MDP $M = \langle \mc{S}, \mc{A}, R, T, \gamma \rangle$ as follows: add a single state $s_i$ to $\mc{S}$ for each element $u_i \in \mc{U}$. Also add two more states, an initial state $s_0$ and terminal state $s_T$. There are two discrete actions $\mc{A} = \{a_1,a_2\}$ and, recalling the correspondence between states and elements of $\mc{U}$, we define rewards as $R(s_i, a_1) = u_i$ and $R(s_i,a_2) = 0$,  $\forall i \in [N]$. All rewards for the initial and terminal states are $0$. We define the transition function as $T(s_i, a, s_T) = 1, \forall a \in \mc{A}, i \in [N]$ such that taking either action from a state corresponding to an element of $\mc{U}$ leads directly to the terminal state. Moreover, $T(s_0,a,s_i) = \frac{1}{N}, \forall a \in \mc{A}, i \in [N]$ which, from the initial state, transitions to a state matching an element of $\mc{U}$ uniformly at random, under either action. Finally, let the discount factor $\gamma = 1$. This MDP is pictured in \autoref{fig:ss_mdp_np_hard}.\\

Consider any deterministic policy $\pi:\mc{S} \rightarrow \mc{A}$ for MDP $M$ and construct a subset $X = \{u_i : u_i \in \mc{U}, s_i \in \mc{S} \setminus \{s_0,s_T\}, \pi(s_i) = a_1 \} \subseteq \mc{U}$ consisting of all states in $\mc{S}$ corresponding to elements of $\mc{U}$ where policy $\pi$ takes action $a_1$. Examining the Bellman equation for the value function induced by policy $\pi$, we see that:
\begin{align*}
    V^\pi(s_0) &= R(s_0, \pi(s_0)) \\
    &\qquad + \gamma \sum\limits_{s' \in \mc{S}} T(s' \mid s_0,\pi(s_0)) V^\pi(s') \\
    &= \frac{1}{N}\sum\limits_{i=1}^N V^\pi(s_i) \\
    &= \frac{1}{N}\sum\limits_{i=1}^N R(s_i,\pi(s_i)) \\
    &= \frac{1}{N}\sum\limits_{i=1}^N u_i \indic\left\{\pi(s_i) = a_1\right\} \\
    &= \frac{1}{N} \sum\limits_{x \in X} x
\end{align*}

From this, we see that computing $N \cdot V^\pi(s_0)$ for any policy $\pi$ yields the sum of the subset induced by $\pi$. Let $\Pi_M = \{\pi: \mc{S} \ra \mc{A}\}$ be the policy class for MDP $M$ and note that $|\Pi_M| = |\mc{A}|^{|\mc{S}|} = 2^{N+2} = 4|\mc{P}(\mc{U})|$ where $\mc{P}(\mc{U})$ denotes the power set of $\mc{U}$. We see that the action choices made by each policy $\pi$ at states $\{s_1,\ldots,s_N\}$ encode a unique subset of $\mc{U}$. Assume we have access to a polynomial-time algorithm for computing $\BPD_\tau(M)$. For a small constant $\eps > 0$, consider computing $(\BPD_{\frac{1}{N}(t+\eps)}(M) - \BPD_{\frac{1}{N}(t-\eps)}(M))$ where
\begin{align*}
    &\BPD_{\frac{1}{N}(t+\eps)}(M) \\
    &= \frac{1}{|\Pi_M|} \sum_{\pi \in \Pi_M}\indic\left\{ V^{\pi}(s_0) \leq \frac{1}{N}(t+\eps)\right\} \\
    &= \frac{1}{2^{N+2}} \sum_{\pi \in \Pi_M} \indic\left\{ NV^{\pi}(s_0) \leq (t+\eps)\right\} \\
    &= \frac{1}{4|\mc{P}(\mc{U})|} \sum_{\pi \in \Pi_M} \indic\left\{ \sum\limits_{x \in X_\pi} x \leq (t+\eps)\right\} \\
    &= \frac{4}{4|\mc{P}(\mc{U})|} \sum_{X_\pi \in \mc{P}(\mc{U})} \indic\left\{ \sum\limits_{x \in X_\pi} x \leq (t+\eps)\right\} \\
    &= \frac{1}{|\mc{P}(\mc{U})|} \sum_{X_\pi \in \mc{P}(\mc{U})} \indic\left\{ \sum\limits_{x \in X_\pi} x \leq (t+\eps)\right\} \\
\end{align*}

where the set $X_\pi$ is constructed for each policy $\pi \in \Pi_M$ as described above. Consequently, we've shown that $\BPD_{\frac{1}{N}(t+\eps)}(M)$ equals the fraction of subsets of $\mc{U}$ whose total element-wise sum is upper bounded by $(t+\eps)$; a similar statement follows for $\BPD_{\frac{1}{N}(t-\eps)}(M)$. With an arbitrarily small constant $\eps > 0$, we can use our polynomial-time algorithm for the BPD problem to compute $\indic \left\{(\BPD_{\frac{1}{N}(t+\eps)}(M) - \BPD_{\frac{1}{N}(t-\eps)}(M)) > 0\right\}$ and determine the existence of a subset $X \subseteq \mc{U}$ such that $\sum_{x \in X} x = t$, with arbitrarily high accuracy. Thus, we arrive at a polynomial-time algorithm for the decision version of the \textsc{SubsetSum} problem and the BPD problem must be NP-hard.\qedhere
\end{dproof}

\begin{dproof}[\autoref{prop:bpd_add_approx}]
Our algorithm is as follows. We sample $k = \Theta\left(\frac{1}{\eps}\log n \right)$ policies uniformly at random for some sufficiently large hidden constant. We return as our estimate of the good policy density $\widehat{\GPD}_{\tau}(M)$ the fraction of these policies that achieve value at least $\tau$ on $s_0$. Recall the Chernoff-Hoeffding bound which states that given $X := \sum_{i=1}^k X_i$ where each $X_i$ is an i.i.d. Bernoulli with $a_i \leq X_i \leq b_i$ with probability $1$ we have
\begin{align*}
    \Pr[|X - E[X]| \geq \eps] \leq 2 \exp\left(\frac{-2 \eps^2}{\sum_{i=1}^k(b_i - a_i)} \right).
\end{align*}
We apply this bound where each $X_i$ corresponds to one of our samples and is $1/k$ if in the sampled policy we have value at least $\tau$ on $s_0$ and $0$ otherwise. Notice that $a_i = 0$ and $b_i = 1/k$ for every $i$. Thus, we have $\GPD_{\tau}(M) = E[X]$ and $\widehat{\GPD}_{\tau}(M) = X$ and so
\newcommand{\poly}{\mathrm{poly}}
\begin{align*}
    \Pr[| \widehat{\GPD}_{\tau}(M) &- \GPD_{\tau}(M)| \geq \eps] \\
    &\leq 2 \exp \left(\frac{-2 \eps^2}{1/k^2} \right)\\
    & = 2 \exp(-\Theta(\log n)) \\
    & = \frac{1}{\poly(n)}.
\end{align*}
as desired.\qedhere
\end{dproof}

\begin{dproof}[\autoref{thm:bpd_poly_approx}]
    For this proof we will assume that every state has exactly two actions $a_1$ and $a_2$. The proof easily generalizes to general action spaces.

	In many ways the principle challenge in accurately estimating the good policy density is overcoming the ``needle in the haystack'' situation. In particular, if a constant fraction of all policies have value at least $\tau$ on the initial state $s_0$ then we can simply repeatedly sample a policy uniformly at random and then use the fraction of policies with value at least $\tau$ on $s_0$ among all sampled policies as our estimate of the good policy density. Standard Chernoff-bound-type arguments will show that such an estimate is very accurate with only polynomially-many samples.
	
	However, that's assuming a whole lot of needles! If, on the other hand, a very small number of policies have value at least $\tau$ on $s_0$ then no such sampling strategy could possibly work efficiently as one would never sample a policy with value at least $\tau$ on $s_0$.
	
	As an alternative we will utilize a general strategy first proposed by \cite{jerrum1989approximating} to estimate the permanent of a matrix (which is computationally equivalent to estimating the number of matchings in a bipartite graph).
	
	We will sketch this strategy in terms of MDPs. We will construct a sequence of MDP-like things $(M_i)_{i=1}^k$ where $M_k$ is the original MDP in which we would like to estimate the good policy density. Let $T_i$ be the policies in $M_i$ with value at least $\tau$ on $s_0$ and let $N_i := |T_i|$. We can estimate the good policy density by dividing $N_k$ by the total number of policies. To estimate $N_k$ we observe by a simple telescoping multiplication that
	\begin{align*}
	N_k = \frac{N_k}{N_{k-1}} \cdot \frac{N_{k-1}}{N_{k-2}} \cdot \ldots \cdot \frac{N_1}{N_0} \cdot N_0. \end{align*}
	
	Thus, to estimate the good policy density it suffices to estimate every $\frac{N_{i+1}}{N_{i}}$ and $N_0$.
	
	Let us suppose that $N_0$ is trivial to estimate given the way we constructed our sequence. Why should we expect that estimating the ratio $\frac{N_{i+1}}{N_{i}}$ is any easier than just estimating $N_k$? Well suppose that our sequence satisfied the following two smoothness properties
	\begin{enumerate}
		\item $T_{i+1} \subseteq T_i$
		\item $\frac{1}{2} \leq \frac{N_{i+1}}{N_i}$
	\end{enumerate}
	The existence of such smoothness properties suggests the following algorithm for estimating $\frac{N_{i+1}}{N_i}$: sample polynomially many policies independently and uniformly at random from $T_{i+1}$ and estimate $\frac{N_{i+1}}{N_i}$ as the proportion of sampled policies that are also in $T_i$. Since $\frac{1}{2} \leq \frac{N_{i+1}}{N_i}$ and $T_{i+1} \subseteq T_i$ we know that a constant fraction of the time when we take a sample and perform our check the sampled policy will indeed be in $T_{i+1}$; by standard Chernoff bound-type arguments mentioned above we will attain a good estimate of our ratio. Thus, our smoothness property has greatly increased the proportion of needles to hay and in this way allows us to use a Chernoff bound.
	
	Thus, to realize this strategy we must construct a sequence $(M_i)_{i=1}^k$ satisfying the above smoothness property where additionally $N_0$ is efficiently estimable.
	
	We proceed to discuss how to construct the aforementioned sequence. Our sequence will be $M_0, M_1, \ldots, M_{|S|}$ where $M_{|S|}$ is the MDP in which we would like to estimate the number of policies whose value on $s_0$ is at least $\tau$ , i.e.\ $N_{|S|}$.
	
	Our $M_i$ for $i < |S|$ won't exactly be an MDP per se. Rather, $M_i$ can be intuitively thought of as an MDP in which the agent only has control over the first $i$ states (for some canonical ordering over states) and in the remaining states the agent is forced to behave optimally.

	More formally, let $\Pi_i^\pi := \{\pi' : \pi'(s_j) = \pi(s_j) \text{ for } j \leq i \}$ be all policies consistent with $\pi$ on the first $i$ states. Since the Bellman equations have a fixed point, by viewing $\Pi_i^\pi$ as all policies in an MDP in which every state $s_j$ for $j \leq i$ has only the action taken by $\pi$ available in it, we know that there is a policy in $\Pi_i^\pi$ whose value is greater than or equal to that of all other policies in $\Pi_i^\pi$ on all states. Let $(\pi^*_i \mid \pi) \in \Pi_i^\pi$ be any such optimal completion of $\pi$. We can now define the ``value'' of policy $\pi$ in a state in $M_i$ as follows.
	
	\begin{definition}[$V_i^\pi(s)$]
		$V_i^\pi(s) := V^{(\pi^*_i \mid \pi)}(s)$. That is, the value of $\pi$ in state $s$ in $M_i$ is the value of the optimal completion of $\pi$, i.e.\ $(\pi^*_i \mid \pi)$, in this state.
	\end{definition}

	Notice that under this definition $V_{|S|}^\pi(s) = V^\pi (s)$ and so indeed $M_{|S|}$ just is our original MDP.
	
	We can now define the above set $T_i$.
	\begin{definition}[$T_i$]
	Let $T_i := \{\pi : \tau \leq V_i^\pi(s_0)\}$ be all policies that in $M_i$ have value at least $\tau$ on $s_0$. We call $T_i$ the set of all $\tau$-quality $i$-controlled policies.
	\end{definition}
	
	We begin by noting that the number of $\tau$-quality $i$-controlled policies is trivial to compute in $N_0$ since no matter what the policy of the agent, the agent is forced to behave optimally on all states. Also, we may assume that $V^*(s_0) \geq \tau$ since otherwise the good policy density is always $0$ and can be trivially computed as such.
	\begin{lemma}
	$N_0 = |A|^{|S|}$ if $V^*(s_0) \geq \tau$.
	\end{lemma}

	We now prove the above two smoothness properties.
	
	\begin{lemma}
	$T_{i+1} \subseteq T_i$.
	\end{lemma}
	\begin{proof}
		Consider a $\pi \in T_{i+1}$. We will show that $\pi \in T_i$. That is, we know that $V_{i+1}^\pi(s) = (\pi_{i+1}^* \mid \pi) \geq \tau$ and we must argue that $V_i^\pi(s) = (\pi_i^* \mid \pi) \geq \tau$ for every state $s$. However, we need only notice that $\Pi_{i+1}^\pi \subseteq \Pi_i^\pi$ since the former is all policies consistent with $\pi$ on the first $i+1$ states and the latter is all policies consistent with $\pi$ on the first $i$ states. Thus, we know that $(\pi_{i+1}^* \mid \pi) \in \Pi_i^\pi$ and so an optimal policy in $\Pi_i^\pi$---namely $(\pi_i^* \mid \pi)$---must achieve at least the same value as $(\pi_{i+1}^* \mid \pi)$ on every state---namely at least $\tau$ on $s_0$.
	\end{proof}

	\begin{lemma}
	$\frac{1}{2} \leq \frac{N_{i+1}}{N_i}$
	\end{lemma}
	\begin{proof}
		To show that $\frac{1}{2} \leq \frac{N_{i+1}}{N_i}$, we will construct an injective function $f : T_{i} \setminus T_{i+1} \to T_{i+1}$. Letting $\lnot a_1 = a_2$ and $\lnot a_2 = a_1$, our function is $f(\pi) = \pi'$ where $\pi'$ is defined as follows:
		\begin{align*}
		\pi' := \begin{cases} \lnot \pi(s_j) & \text{if $j = i + 1$}\\ \pi(s_j) & \text{o/w}\end{cases}
		\end{align*}
		
		That is, $f$ outputs a policy that is identical to $\pi$ but which takes the opposite action as $\pi$ on state $i+1$. To complete our proof we argue that if $\pi \in T_{i} \setminus T_{i+1}$ then $f(\pi) \in T_{i+1}$. 
		
		To do so, it suffices to argue that  $(\pi_i^* \mid \pi) \in \Pi^{f(\pi)}_{i+1}$. Since $(\pi_i^* \mid \pi)$ matches $\pi$ on the first $i$ states and $f(\pi)$ matches $\pi$ on the first $i$ states, we know that $f(\pi)$ also matches $(\pi_i^* \mid \pi)$ on the first $i$ states. However, recall that $(\pi_i^* \mid \pi)$ has value at least $\tau$ on $s_0$, but $(\pi_{i+1}^* \mid \pi)$ does not and so we know that $(\pi_i^* \mid \pi) \not \in \Pi^{\pi}_{i+1}$. It must therefore be the case that $\pi(s_{i+1}) \neq (\pi_i^* \mid \pi) (s_{i+1})$ which is to say that $f(\pi)(s_{i+1}) = (\pi_i^* \mid \pi)$; that is $f(\pi)$ matches $(\pi_i^* \mid \pi)$ on the $(i+1)$th state. But then $f(\pi)$ matches $(\pi_i^* \mid \pi)$ on all of the first $(i+1)$th states and so we know that $(\pi_i^* \mid \pi) \in \Pi^{f(\pi)}_i$.
	\end{proof}

	Thus, we have constructed our sequence satisfying the aforementioned smoothness properties and we can easily compute $N_0$. Assuming uniform sampling access to the $\tau$-quality $i$-controlled policies, standard Chernoff bound arguments show we can estimate each $\frac{N_{i+1}}{N_i}$ up to a multiplicative $(1 \pm \eps)$ using the previously-mentioned strategy in poly-time. By our earlier arguments this suffices to estimate the good policy density up to a multiplicative $(1 \pm \eps)$ in poly-time.\qedhere
\end{dproof}

\section{Experimental Details}

We next provide additional details about the $N$-chain experiment in \autoref{fig:bpd_chain}, and an additional experiment of the same form in the $4\times 3$ Russell \& Norvig grid world \cite{russell1995modern}.

\paragraph{$\bm{N}$-Chain Details.} As discussed, the results pictured in \autoref{fig:bpd_chain} highlight the change of $\BPD_\tau(M)$ for different choices of $\tau$ and $M$. We compute the $\BPD_\tau(M)$ for each $M$ by exactly computing the start-state value of every deterministic policy, which is feasible given the size of the MDPs. The start state is the left-most state in each chain, and $\gamma$ was set to $0.95$. For each MDP, we compute $\BPD_\tau(M)$ for 12 choices of $\tau$, starting from $\textsc{VMin}$ and incrementing by $\frac{\textsc{VMax} - \textsc{VMin}}{12}$ up to $\textsc{VMax}$. In general, we anticipate that most of the interesting action will happen relatively close to $\BPD_\tau(M) \approx 1$, as on most MDPs of relevance the policy space likely only contains a few good policies.

\begin{figure}[t!]
    \centering
    \subfigure[With Lava]{\includegraphics[width=0.45\textwidth]{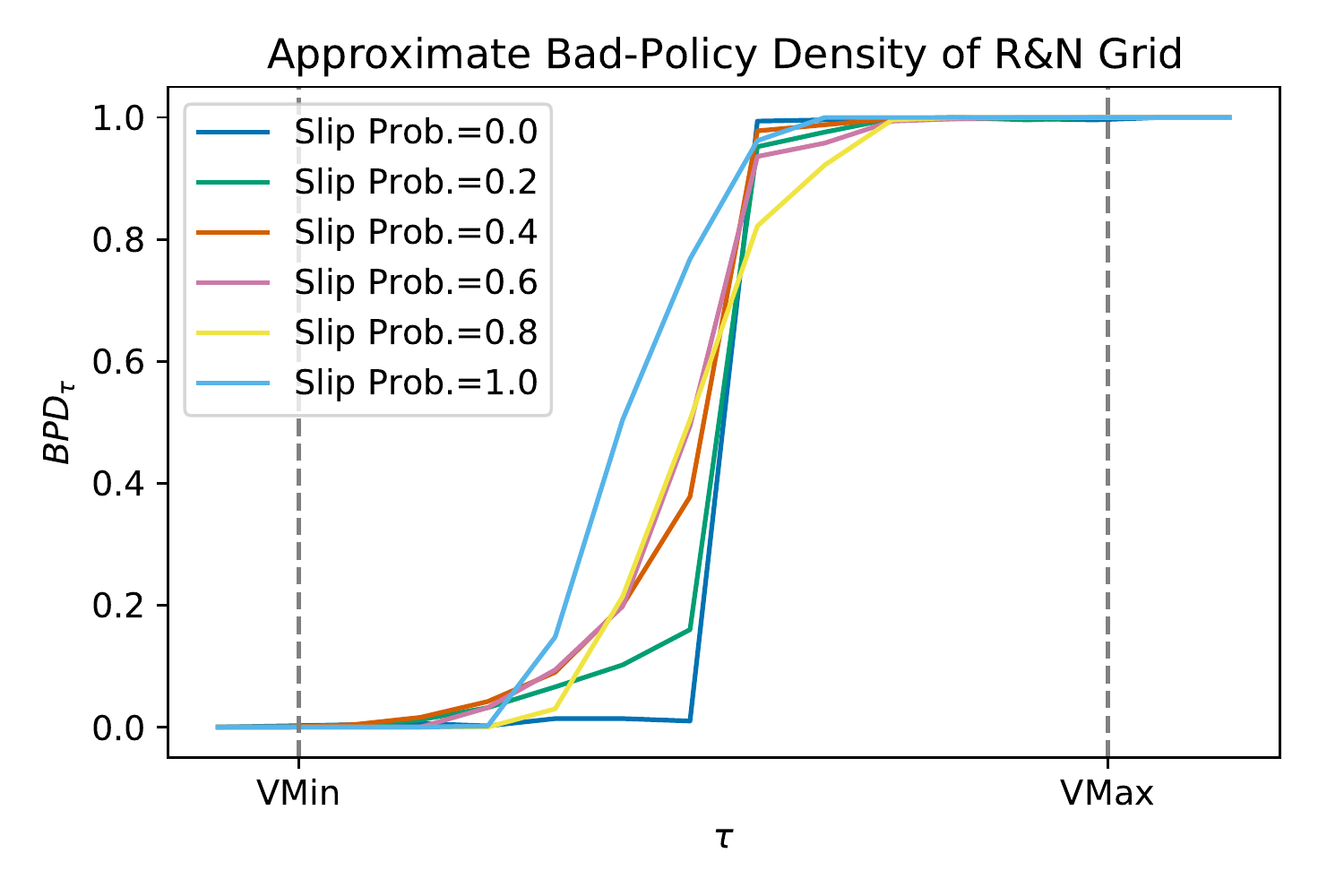}}
    \subfigure[No Lava]{\includegraphics[width=0.45\textwidth]{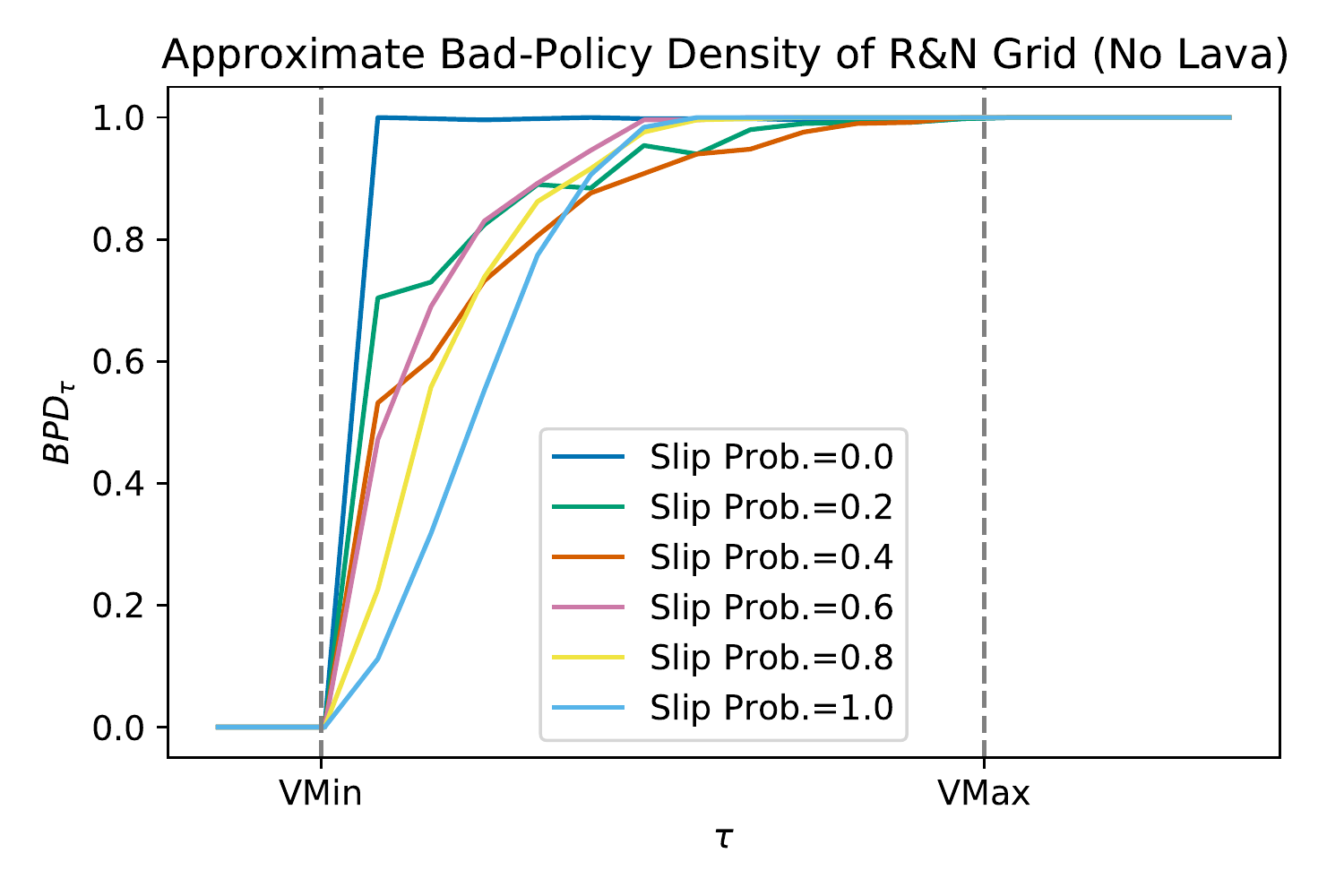}}
    \caption{The BPD of variants of the Russell \& Norvig grid world. On the top plot, we visualize the BPD of the traditional grid that includes a terminal lava state at $(4,2)$, in which the agent receives $-1$ reward. On the bottom, we visualize the BPD of this grid world without the lava state---instead, $(4,2)$ is a non-terminal empty cell that yields $0$ reward. We vary the slip probability in each problem from $0.0$ up to $1.0$.}
    \label{fig:bpd_rn_grid}
\end{figure}

\paragraph{Grid World.} We next experiment with the Russell \& Norvig grid world \cite{russell1995modern}, a 4$\times$3 grid world containing a wall at $(2,2)$, a terminal goal at $(4,3)$ that awards $+1$ upon entering, and a terminal lava pit at $(4,2)$ that awards a $-1$ upon entering. The agent starts at $(1,1)$ and can move in each of the four cardinal directions. We introduce a slip probability $\epsilon$ in which, for each $(s,a)$ pair there is an $\epsilon$ probability that the agent will execute an action uniformly at random with probability each time step. In our experiment, we inspect the $\BPD_\tau(M)$ for different settings of $\epsilon$ from $0.0$ up to $1.0$. 
Given the size of the policy space of even this small grid world, we here use the additive approximation method to estimate $\BPD_\tau(M)$ for each $\tau$ and $M$ combination. We sample 500 policies at random and evaluate them to form our estimate of $\BPD_\tau(M)$. Note that the additive approximation method is only suitable for MDPs where $\BPD_\tau(M)$ is non-negligibly distant from one or zero---otherwise it is likely that estimate will always yield zero or one, despite the existence of several good or bad policies. For this reason, we emphasize the significance and important of \autoref{thm:bpd_poly_approx}, though note that there is still a remaining step to make this approach practical.

%
Result are presented in \autoref{fig:bpd_rn_grid}. On the top, we show the estimated BPD of the grid that contains the lava cell for each value of $\tau$. On the bottom, we show the estimated BPD of the grid with the lava cell replaced by an empty cell (so the agent receives +0 and does not terminate at $(4,2)$). We again vary $\tau$ from \textsc{VMin} up to \textsc{VMax} in increments of $\frac{\textsc{VMax} - \textsc{VMin}}{12}$. 

First note that we lose monotonicity because we are \textit{estimating} $\BPD_\tau(M)$ rather than computing it exactly. Next, observe that for any middling choice of $\tau$ (less than roughly $\frac{\textsc{VMax} - \textsc{VMin}}{2}$), we find the problem to be trivially easy for lower choices of slip probability when lava is present. This is because only a few policies make their way to the lava cell---thus, because value range includes the costly policies that go directly into the lava, the bad policies are only those that move directly to the lava. Even the uniform random policy only has a small probability of arriving at the lava, and thus even the light blue curve (when the slip probability is 1.0) decays toward 0 for lower values of $\tau$. Conversely, once the $\tau$ threshold requires the agent actually reaches the goal in a timely fashion, nearly all policies are considered below optimal, and thus the problem becomes more difficult. While these dramatic swings are interesting and anticipated, we suspect that for MDPs of practical interest, $\log_{\BPD_\tau(M)}$ may be the more interesting property to inspect as it will help distinguish when there are a handful of good policies as opposed to several handfuls.


In the case with no lava, we first note that the range $[\textsc{VMin},\textsc{VMax}]$ is in fact different, as there is no longer a negative reward in the MDP. Hence, we find that when the threshold increases past $\tau=0$, many policies are quickly taken to be bad, as they do not reach the goal. The higher the slip probability, the more policies that can be considered good, as the stochasticity of the environment pushes more policies toward \textit{eventually} finding the goal. In contrast, we note that when the slip probability is zero, there are only a select few (out of $|\mc{A}|^{|\mc{S}|} = 4,194,304$) policies that ever reach the goal, which is why we see the darker blue curve rise quickly around $\tau = \textsc{VMin}$. 

As a final note, we observe that comparing the value \textit{intervals} between the lava and no-lava cases gives the impression that the MDP without lava is harder, in general. This is primarily due to the fact that there are few values of $\tau$ for which any policy might be considered good. In contrast, when lava is present, \textsc{VMin} is considerably lower. We suggest these facts may be important when considering generalizations of the BPD that avoid explicit dependence on $\tau$: Such generalizations may need to account for the \textit{range} of values policies can take on. One route to incorporating such considerations is to move to a probabilistic view of the $\BPD_\tau(M)$, in which we assess the \textit{probability} of sampling a policy below a particular value threshold.

\end{document}